%% file: main.tex
\documentclass[10pt,twocolumn,letterpaper]{article}

\usepackage{cvpr}              

\input{preamble}

\definecolor{cvprblue}{rgb}{0.21,0.49,0.74}
\usepackage[pagebackref,breaklinks,colorlinks,allcolors=cvprblue]{hyperref}
\usepackage[accsupp]{axessibility}
\usepackage{multirow}
\usepackage{graphicx}
\usepackage{booktabs}
\usepackage{adjustbox}
\usepackage{rotating}
\usepackage{makecell}
\usepackage{amsthm}
\usepackage{marvosym}
\usepackage{algorithm}
\usepackage{algpseudocode}
\usepackage[table]{xcolor} 

\newtheorem{proposition}{Proposition}[section]
\newtheorem{assumption}{Assumption}[section]

\def\paperID{***} 
\def\confName{CVPR}
\def\confYear{2026}

\title{The Golden Subspace: Where Efficiency Meets Generalization \\ in Continual Test-Time Adaptation}

\author{
Guannan Lai$^{1,2}$, 
Da-Wei Zhou$^{1,2}$, 
Zhenguo Li$^{3,4}$, 
Han-Jia Ye$^{1,2}$\textsuperscript{(\Letter)} \\
$^{1}$ School of Artificial Intelligence, Nanjing University \\
$^{2}$ National Key Laboratory for Novel Software Technology, Nanjing University \\
$^{3}$ Hong Kong University of Science and Technology \quad
$^{4}$ Frontier Robotics\\
{\tt\small \{laign, zhoudw, yehj\}@lamda.nju.edu.cn, zhenguol@gmail.com}
}

\begin{document}
\maketitle
\footnotetext[2]{Correspondence to: Han-Jia Ye (yehj@lamda.nju.edu.cn)}

\input{sec/0_abstract}    
\input{sec/1_intro}
\input{sec/2_rw}
\input{sec/3_ana}
\input{sec/4_method}

\input{sec/5_exp}
\input{sec/6_con}

\clearpage

\section*{Acknowledgments}
This work is partially supported by the National Key R\&D Program of China (2024YFE0202800), the NSFC (62522605, 62506160), the Basic Research Program of Jiangsu (BK 20251251), JSTJ-2025-147, and the Fundamental and Interdisciplinary Disciplines Breakthrough Plan of the Ministry of Education of China (No. JYB2025XDXM118). We sincerely thank Daniel Beaglehole, one of the originators of RFM, for insightful discussions and constructive suggestions that helped clarify the use of AGOP in CTTA. We also thank Hao-Run Cai and Xiangkun Wang for helpful discussions.

{
    \small
    \bibliographystyle{ieeenat_fullname}
    \bibliography{main}
}

\input{sec/X_suppl}

\end{document}

%% file: preamble.tex
\usepackage{soul}

%% file: sec/0_abstract.tex
\begin{abstract}
Continual Test-Time Adaptation (CTTA) aims to enable models to adapt online to unlabeled data streams under distribution shift without accessing source data. Existing CTTA methods face an efficiency–generalization trade-off: updating more parameters improves adaptation but severely reduces online inference efficiency. An ideal solution is to achieve comparable adaptation with minimal feature updates; we call this minimal subspace the \emph{golden subspace}. We prove its existence in a single-step adaptation setting and show that it coincides with the row space of the pretrained classifier. To enable online maintenance of this subspace, we introduce the sample-wise Average Gradient Outer Product (AGOP) as an efficient proxy for estimating the classifier weights without retraining. Building on these insights, we propose \textbf{Guided Online Low-rank Directional adaptation (GOLD)}, which uses a lightweight adapter to project features onto the golden subspace and learns a compact scaling   vector while the subspace is dynamically updated via AGOP. Extensive experiments on classification and segmentation benchmarks, including autonomous-driving scenarios, demonstrate that GOLD attains superior efficiency, stability, and overall performance. Our code is available at \href{https://github.com/AIGNLAI/GOLD}{https://github.com/AIGNLAI/GOLD}.
\end{abstract}

%% file: sec/1_intro.tex
\section{Introduction}

In real-world applications, models are often required to perform inference under continuously shifting data distributions \cite{boudiaf2023search,fan2024dynamic,kundu2020universal}, such as weather and illumination changes in autonomous driving, scene variations in video stream analysis, and device heterogeneity in medical imaging \cite{koh2021wilds,yang2020fda,yang2021generalized}. These scenarios demand models that can adapt online during test time to unknown and gradually changing target domains, without re-accessing source data or performing offline retraining. This problem setting, known as \textbf{Continual Test-Time Adaptation (CTTA)}, aims to enable models to maintain robust performance in dynamic environments \cite{wang2022continual,cui2025continual,park2025hybrid}.
\begin{figure}
    \centering
    \begin{subfigure}{0.235\textwidth}
        \centering
        \includegraphics[width=\linewidth]{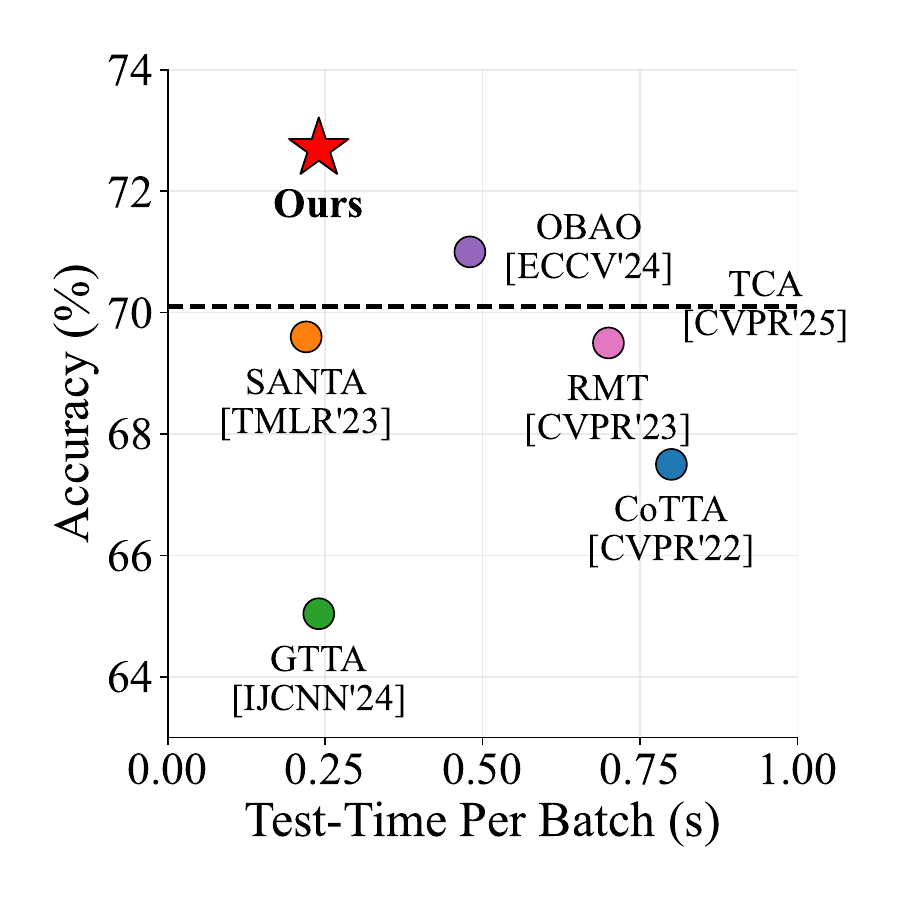}
        \caption{Efficiency and Accuracy}
        \label{fig:time}
    \end{subfigure}
    \begin{subfigure}{0.235\textwidth}
        \centering
        \includegraphics[width=\linewidth]{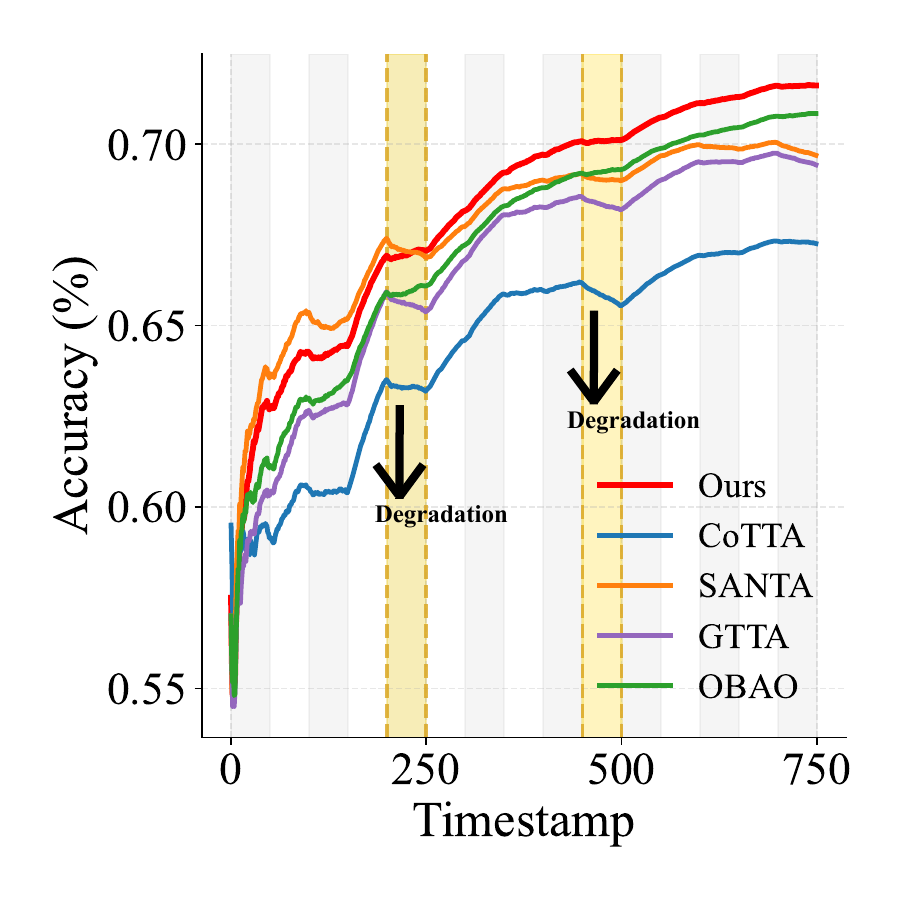}
        \caption{Generalization}
        \label{fig:acc}
    \end{subfigure}
        \caption{Results on the CIFAR100-C dataset. \emph{\textbf{Left:}} performance–efficiency distribution of evaluated methods (TCA \cite{ni2025maintaining} is closed-source and shown as a line). \emph{\textbf{Right:}} accuracy over the continuous test-time adaptation process; shaded background bands indicate periods of domain shift. In the \textbf{golden-shaded} intervals, prior methods exhibit significant performance degradation, whereas our method maintains generalization.}

    \label{fig:in}
\end{figure}
In practical deployments, CTTA faces a trade-off between efficiency and generalization: better generalization typically requires more complex models and more intensive computation, reducing run-time efficiency. Existing approaches typically update model parameters at test time using self-supervised objectives, entropy minimization, or adaptive batch-normalization statistics \cite{liu2024vida, brahma2023probabilistic}. While such updates can provide short-term gains, they substantially increase computational cost and tend to amplify pseudo-label noise and induce parameter drift, undermining the generalization of the continual adaptation process and ultimately degrading long-term performance \cite{chen2024pg, tan2024less, wang2024continual}. As shown in \cref{fig:acc}, when a new domain arrives, existing methods fail to adapt quickly and exhibit sudden performance degradation, indicating an inability to preserve generalization during adaptation.

Ideally, our goal is to achieve the desired change in model output within the feature subspace (to ensure generalization) while keeping the update magnitude minimal (to ensure efficiency). 
We define this subspace as the \textbf{Golden Subspace}. We first consider a single-step adaptation scenario, derive its analytical formulation, and verify the existence of the golden subspace. Our analysis reveals that this subspace is essentially composed of the row-space of the classifier weights, which can be obtained via eigenvalue decomposition of the classifier weight matrix.
\textbf{However}, a new challenge arises: the classifier weights encode predominantly source-domain information, and continuously updating them at test time reintroduces the same problems of high computational cost and structural degradation. \textbf{We therefore ask:} how can the classifier be endowed with target-domain information without retraining its weight?

Motivated by recent advances on Average Gradient Outer Product (AGOP) representations \cite{radhakrishnan2025linear, beaglehole2024average,mallinar2024emergence,beaglehole2025xrfm,beaglehole2026toward}, we find that AGOP computed from high-confidence samples provides a faithful online proxy for the classifier's parameter inner-product structure. Specifically, the golden subspace estimated via AGOP converges over time to the subspace obtained by single-step least-norm adaptation, which guarantees that models built on this estimator retain \textbf{strong generalization}. Empirically, AGOP matrices exhibit low effective rank, implying that the golden subspace is low-rank in practice and therefore amenable to highly \textbf{efficient} online maintenance and adaptation.

Building on these insights, we propose \textbf{Guided Online Low-rank Directional adaptation (GOLD)}, a method for efficient continuous test-time adaptation. GOLD maintains a lightweight matrix, which is initialized from the classifier weight inner-products and is updated online using the AGOP computed from high-confidence test samples. We periodically perform an eigendecomposition of it to extract the current low-rank \emph{golden subspace}. Features from a frozen backbone are projected onto this subspace, and a compact scaling vector is learned to re-scale the projected coordinates for adaptation.
This design introduces only a small number of external parameters and updates a very limited internal parameter subset, thereby intrinsically constraining parameter drift while providing sufficient capacity to refine outputs. As shown in \Cref{fig:in}, GOLD attains the best performance with minimal adaptation cost (\Cref{fig:time}) and preserves high generalization ability across domain shifts (\Cref{fig:acc}).

Our main contributions are summarized as follows:

\begin{itemize}
  \item We formalize the \textit{golden subspace} as the minimal feature-update subspace induced by the classifier and show that it admits a low-rank characterization. We further develop an AGOP-based online estimator that tracks this subspace during test-time adaptation.
  
  \item We introduce GOLD, a source-free online adaptation framework that projects features onto the golden subspace and learns a lightweight scaling vector. GOLD achieves efficient adaptation with minimal  updates.
  
  \item Extensive experiments on \textbf{classification} and \textbf{segmentation} benchmarks demonstrate that GOLD achieves SOTA performance with a minimal computational footprint.
\end{itemize}

%% file: sec/2_rw.tex
\section{Related Work}
\label{sec:rw}

Continual Test-Time Adaptation (CTTA) \cite{liu2024continual,gan2023decorate} extends test-time adaptation to dynamic environments \cite{lai2025order,lai2026the}, where a source model must adapt continuously to a stream of unlabeled and non-stationary target data without revisiting source samples \cite{luo2019taking,miao2024unified}. 

The pioneering CoTTA \cite{wang2022continual} introduced a teacher student framework with weight averaging and stochastic weight restoration to alleviate error accumulation and forgetting.
Building on this idea, PETAL \cite{brahma2023probabilistic} proposed a data-driven parameter recovery mechanism that regularizes model updates toward the source parameter configuration to enhance robustness.
RMT \cite{dobler2023robust} and SANTA \cite{chakrabarty2023santa} enforced feature-level consistency with the source model through contrastive objectives, while DSS \cite{wang2024continual} improved reliability by filtering unsafe pseudo-labels during adaptation.
More recently, TCA \cite{ni2025maintaining} maintained inter-class stability by preserving topological consistency among class representations during domain shifts.
Despite these advances, most existing methods still rely on global or feature-level updates of the entire network \cite{tan2025uncertainty, zhao2025d, chu2025remember}, which inevitably leads to a trade-off between efficiency and generalization.

%% file: sec/3_ana.tex
\section{Preliminary}

\subsection{Problem Definition}

We study the CTTA setting \cite{wang2025search,liang2025comprehensive}, 
where a model $f_{\theta}=h_{\psi}\circ g_{\phi}$ is pretrained on a labeled source domain 
$\mathcal{D}_{\mathrm{src}}=\{(x_i^{\mathrm{src}},y_i^{\mathrm{src}})\}_{i=1}^{N_{\mathrm{src}}}$ 
and deployed to an unlabeled target data stream 
$\mathcal{D}_T = \{\mathcal{X}_1, \mathcal{X}_2, \ldots, \mathcal{X}_T\}$.
Each batch $\mathcal{X}_t = \{x_i^{(t)}\}_{i=1}^{N_t}$ 
is drawn from a target distribution $p_t(x)$ that may evolve over time, 
i.e., $p_t(x) \neq p_{t+1}(x)$.

During deployment, the model encounters the data stream sequentially. 
At each step $t$, it receives the current batch $\mathcal{X}_t$, 
produces predictions $\hat{y}^{(t)} = g_{\theta_t}(\mathcal{X}_t)$, 
and performs online adaptation using only the current unlabeled data, 
without any access to the source domain or previous target samples. 

Two key factors determine the practicality of such continual adaptation: \textbf{efficiency} and \textbf{generalization}.  
For example, in autonomous driving systems, the model must adapt rapidly and robustly within a limited time window to ensure reliable perception and decision-making under shifting environments.

\subsection{Does the Golden Subspace Exist?}
\label{sec:why_directions}

An ideal \textit{golden subspace} achieves the required change in model outputs (ensuring generalization) while minimizing the amount of training (ensuring efficiency). \textbf{We first consider a simple yet insightful single-step adaptation case.}

Given a frozen pretrained classifier \(W\in\mathbb{R}^{C\times L}\) placed on top of a feature extractor. Let the pre-adaptation features for a test batch be \(F\in\mathbb{R}^{B\times L}\), and suppose we aim to realize a desired output correction \(\Delta Y\in\mathbb{R}^{B\times C}\) after adaptation. Seeking the smallest feature-space change (in Frobenius norm) that achieves this correction leads to the least-norm solution
\begin{equation}
    \Delta F^\star = \Delta Y (W^\top)^{\dagger},
    \label{eq: 1}
\end{equation}
where \((\cdot)^{\dagger}\) denotes the Moore–Penrose pseudoinverse.

This algebraic relation directly leads to a rank constraint:
\begin{equation}
    \operatorname{rank}(\Delta F^\star) \le \operatorname{rank}\!\big((W^\top)^{\dagger}\big) = \operatorname{rank}(W^\top W).
    \label{eq: 2}
\end{equation}
This implies that the rank of the golden subspace is constrained by the rank of the classifier weights, making it inherently low-rank. In neural networks, the number of classes is typically small, so the rank of \(W\) is bounded by the number of classes. This means that only a few classifier-related directions are sufficient to modify the model’s predictions for an entire batch, without the need to explore arbitrary high-dimensional perturbations. 

To further interpret this result, consider the singular value decomposition \(W^\top = V \Sigma U^\top\). Substituting into \Cref{eq: 1} gives
$\Delta F^\star = V \Sigma^{\dagger} U^\top \Delta Y$,
showing that the golden subspace is restricted to the subspace spanned by the principal eigenvectors of the classifier. The classifier implicitly defines a set of directions in the feature space to which the output is most sensitive.

This observation confirms the \textbf{existence} of the previously proposed golden subspace. It suggests that we can obtain the golden subspace by performing \textit{eigenvalue decomposition} on the classifier weights. Such a constraint naturally reduces the risk of error amplification caused by noisy pseudo-labels, mitigates catastrophic forgetting by preventing uncontrolled parameter drift, and ultimately leads to a more efficient CTTA process.

\subsection{How to Update the Golden Subspace?}

While \Cref{eq: 1} shows that the golden subspace lies in \(\operatorname{span}(W^\top)\), this fact alone is insufficient in the CTTA setting. The matrix \(W\) encodes a global weight configuration learned on the source domain, whereas the target-domain sample distribution, class confidences, and noise evolve over time. \textbf{Hence the key question is how to update the golden subspace online at test time so that it reflects target-domain semantics}. Retraining the classifier weights is impractical and inefficient, and accumulated errors can destroy the weight structure. Therefore, we require an \textbf{estimator} that dynamically maintains the golden subspace without retraining \(W\), allowing it to implicitly incorporate target-domain information.

We note that, as observed in \cite{radhakrishnan2024mechanism}, for a well-trained neural network \(\hat{f}\) the spectral structure of a layer's weight matrix is closely related to an averaged outer-product of per-sample feature gradients. In particular, for layer \(l\) and its \(i\)-th weight matrix \(W_i^{(l)}\) one can write
\[
   W_i^{(l)\top} W_i^{(l)}
    \propto
    \Bigg(\frac{1}{n}\sum_{p=1}^n \sum_{j=1}^m \nabla_{u_{ij}^{(l)}}\hat{f}(x^{(p)})\,
    \nabla_{u_{ij}^{(l)}}\hat{f}(x^{(p)})^\top \Bigg)^{\alpha}, 
\]
where \(n\) denotes the number of training samples, \(m\) indexes the relevant input subunits (e.g. spatial locations or feature channels) feeding into the \(l\)-th layer, and \(u_{ij}^{(l)}\) is the \(j\)-th input/pre-activation to that layer for unit \(i\). The exponent \(\alpha\) is typically taken to be around \(1/2\) in empirical studies.

\begin{figure}
    \centering
    \begin{subfigure}{0.235\textwidth}
        \centering
        \includegraphics[width=\linewidth]{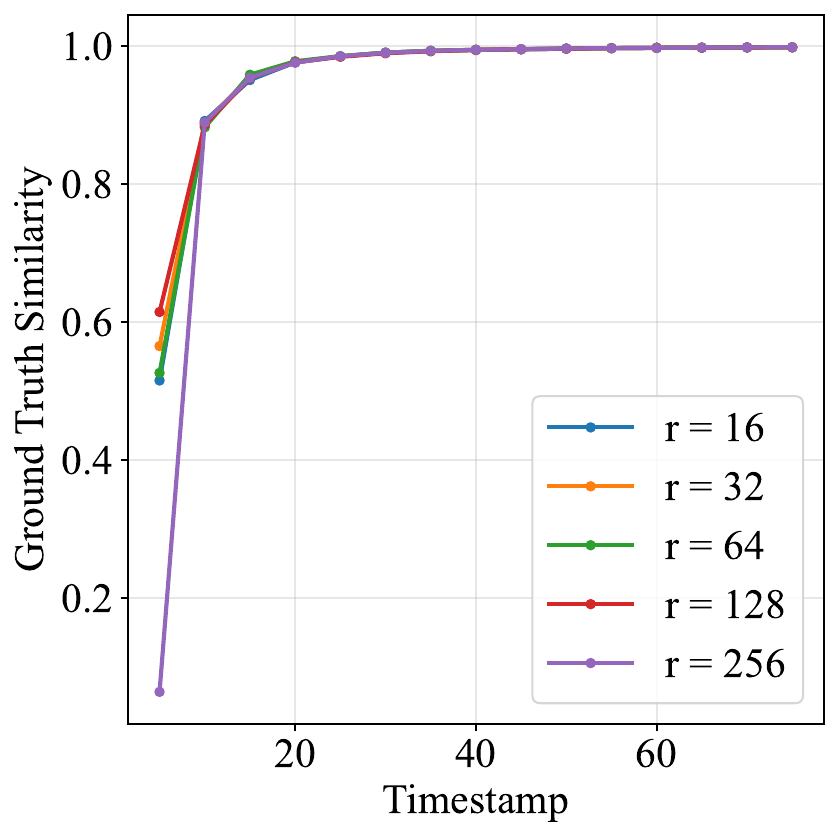}
        \caption{Similarity of golden subspace between AGOP and ground truth.}
        \label{fig:gt}
    \end{subfigure}
    \hfill
    \begin{subfigure}{0.235\textwidth}
        \centering
        \includegraphics[width=\linewidth]{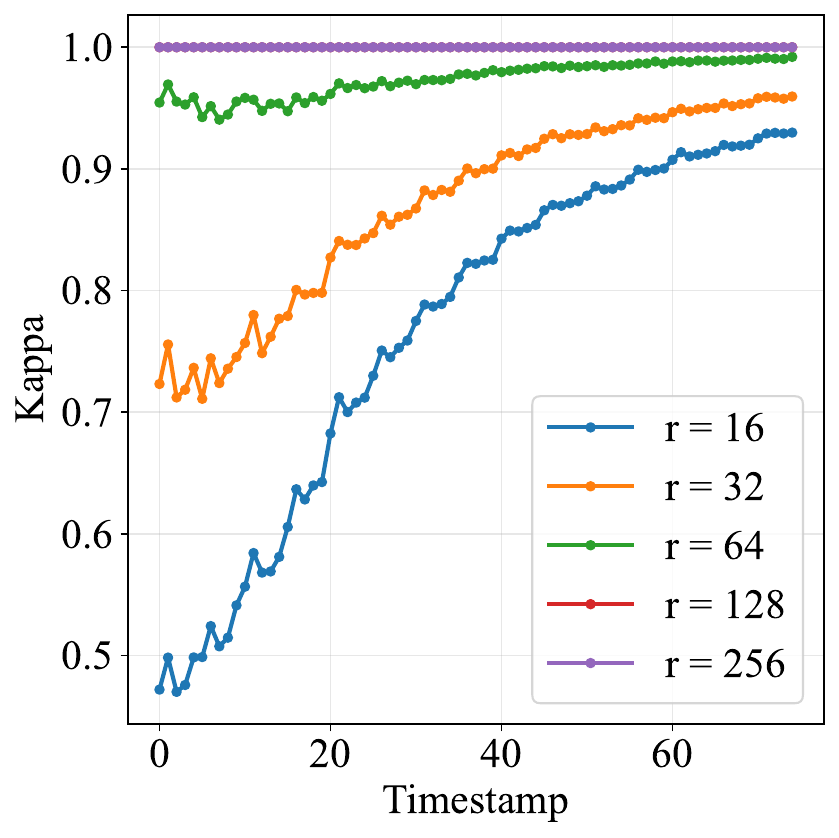}
        \caption{Cumulative spectral energy \(\kappa(k)\) of AGOP.}
        \label{fig:kappa}
    \end{subfigure}
    \caption{(a) Alignment between AGOP-derived subspace and source-derived ground-truth subspace as testing progresses. AGOP rapidly converges and remains highly aligned. (b) AGOP spectrum cumulative energy: the top 64–128 eigenvectors capture over 99\% of the energy, indicating strong low-rank concentration.}
    \label{fig:results}
\end{figure}

Under the CTTA setting, a \textbf{key difference} is that we cannot obtain the labels, so we choose to use pseudo-labels from high-confidence samples for calculation.
To verify the feasibility of this approach, we maintain an online AGOP estimator and perform eigenvalue decomposition to obtain an estimated single-step golden subspace.  
The ground-truth golden subspace is directly computed based on the definition in \Cref{eq: 1}.  
As shown in \Cref{fig:gt}, we visualize the temporal evolution of similarity between the AGOP-derived subspace and the true golden subspace.  
Although the AGOP estimation initially performs poorly after initialization, it quickly converges above \textbf{0.8} and eventually stabilizes above \textbf{0.98}.  
This behavior demonstrates that AGOP serves as an effective estimator of the golden subspace, ensuring \textbf{strong generalization} capability for our method.

To quantify the low-rank nature of the golden subspace, we measure the cumulative spectral energy
$\kappa(k) \;=\; \frac{\sum_{i=1}^k \lambda_i}{\sum_{i=1}^L \lambda_i},$ where \(\{\lambda_i\}\) are the eigenvalues of \(G\) in descending order. The curve in \Cref{fig:kappa} shows that more than \textbf{\(99\%\)} of the spectral energy of \(G\) is captured by only \textbf{\(64\!-\!128\)} eigenvectors, confirming a strong low-rank concentration. This empirical fact justifies guiding continual adaptation with a low-dimensional subspace: doing so \textbf{reduces computational cost} while tightly controlling representational drift.

\begin{figure*}[!t]
    \centering
    \includegraphics[width=0.65\linewidth]{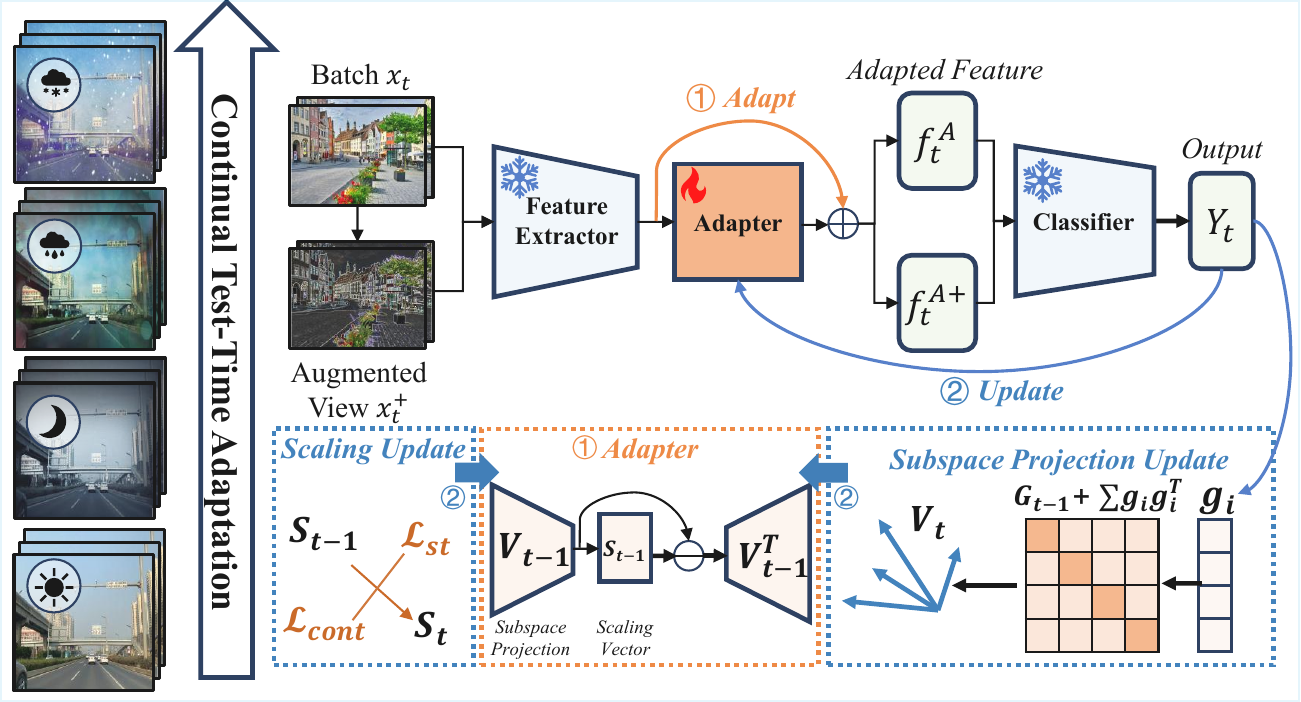}
    \caption{Overview of GOLD: it maintains an online AGOP estimator, from which the top-$r$ eigenvectors \(V_t\) are extracted via eigendecomposition to form the golden subspace. A residual low-rank adapter of rank \(r\) is then applied to learn a lightweight scaling vector \(S_t\). Meanwhile, a combination of self-training loss and prototype-based contrastive loss provides stable supervision.
    }
    \label{fig: method}
\end{figure*}

Finally, we observe that using a moderately low rank yields the best trade-off in practice: lower ranks lead to faster convergence and more efficient adaptation (as seen in Figure~\ref{fig:gt}), but an overly small rank fails to cover all important directions (Figure~\ref{fig:kappa}). Accordingly, all experiments in this work use a moderate subspace dimensionality of 64.

%% file: sec/4_method.tex
\section{Methodology}

Based on the above analysis, we reach two key insights: 1) the golden subspace indeed exists and can be directly obtained via eigendecomposition of the classifier weights; 2) the AGOP computed from test samples provides a powerful approximation of the classifier weight dynamics without retraining. Building on these findings, we propose Guided Online Low-rank Directional adaptation \textbf{(GOLD)}, as illustrated in \Cref{fig: method}.

GOLD maintains a lightweight auxiliary matrix \(G_t\), initialized from the inner-products of classifier weights and updated online using AGOP computed from high-confidence test samples. For each test batch, GOLD operates in two stages: in the \textbf{adapt stage}, backbone features are projected onto the golden subspace and refined via a compact scaling vector \(S\) for lightweight feature adaptation; in the \textbf{update stage}, the golden subspace is continually maintained through AGOP from confident samples, while the scaling vector \(S\) is optimized via self-training and prototype-based contrastive losses, enabling stable and source-free continual adaptation.

\subsection{Pretraining and Prototype Extraction}
\label{sec:4.1}

We first pretrain a base model \(f_{\theta}=h_{\psi}\circ g_{\phi}\) on the labeled source dataset \(\mathcal{D}_{\mathrm{src}}=\{(x_i^{\mathrm{src}},y_i^{\mathrm{src}})\}_{i=1}^{N_{\mathrm{src}}}\) using standard supervised losses (e.g., cross-entropy). After convergence the learned parameters \(\theta^{\ast}=(\phi^{\ast},\psi^{\ast})\) induce a discriminative feature space \(\mathbb{R}^{L}\) in which samples of the same class form compact clusters. The pretrained feature extractor \(g_{\phi^{\ast}}\) is then frozen for subsequent test-time adaptation.
For each class \(c\in\{1,\dots,C\}\) we compute a class prototype as the mean embedding of all source examples of that class:
\[
    P_{c}=\frac{1}{|\mathcal{I}_{c}|}\sum_{i\in\mathcal{I}_{c}} g_{\phi^{\ast}}(x_{i}^{\mathrm{src}})\in\mathbb{R}^{L},\qquad \mathcal{I}_{c}=\{i \mid y_{i}^{\mathrm{src}}=c\}.
\]
All class prototypes are collected into the prototype matrix \(P=[P_{1},P_{2},\dots,P_{C}]^{\top}\in\mathbb{R}^{C\times L}\). The matrix \(P\) serves as a set of \emph{semantic anchors}, namely, class-level reference points that preserve the source-domain semantic geometry. During deployment, we never access source samples. Instead, we only retain the pre-computed class prototypes extracted offline before adaptation, which serve as fixed semantic anchors throughout test-time adaptation.

\subsection{Subspace Projection and Adaptive Rescaling}
\label{sec:4.2}

For the $t$-th test batch we denote the input batch by $x_t\in\mathbb{R}^{B\times\cdot}$ and write the extracted features as
$F_t = g_{\phi^{\ast}}(x_t)\in\mathbb{R}^{B\times L},$
so that each row $f\in\mathbb{R}^L$ of $F_t$ corresponds to one example in the batch. We introduce a subspace projection matrix $V_t\in\mathbb{R}^{L\times r}$ and a scaling vector $S_t\in\mathbb{R}^r$. The $r$ columns of $V_t$ form a basis for a low-dimensional subspace of $\mathbb{R}^L$, and $S_t$ specifies element-wise modulation of the coordinates in that subspace.

Given a single feature $f\in\mathbb{R}^L$, we first project it onto the \textit{golden subspace} by
\[
    u \;=\; V_t^\top f \;\in\mathbb{R}^r,
\]
then apply an element-wise rescaling of the projected coordinates,
\[
    \tilde{u} \;=\; (1 + S_t)\odot u \;\in\mathbb{R}^r,
\]
where $\odot$ denotes the Hadamard (element-wise) product and $1\in\mathbb{R}^r$ is the all-ones vector so that $S_t=0$ yields the identity transform. The modulated coordinates are mapped back to the original feature space and added to the original feature in a residual manner:
\[
    \mathcal{A}(f) \;=\; f + V_t(\tilde{u}-u)
= f + V_t\big( S_t \odot (V_t^\top f)\big).
\]

This residual form preserves the original feature when $S_t=0$, which helps stabilize adaptation.

Applied to the whole batch $F_t\in\mathbb{R}^{B\times L}$, we obtain
\[
    F_t^{\mathrm{adapt}} 
    = F_t + \big(S_t\odot (F_t V_t)\big) V_t^\top.
\]

\noindent At this stage we obtain adapted features for prediction. How should we update the adapter's subspace and the scaling vector? Guided by the derivation in \Cref{sec:why_directions}, we propose an AGOP-based projection update in \Cref{sec:4.3} and a scaling-vector update driven by self-training and prototype-based contrastive losses in \Cref{sec:4.4}.

\subsection{AGOP-based Subspace Projection Update}
\label{sec:4.3}

Subspace projection aims to identify the directions in feature space that most require adaptation. 
Following the analysis in \Cref{sec:why_directions}, we initialize an auxiliary matrix \(G_0 = W^\top W\), which reflects the global geometry of the source classifier. 
To enable adaptation under shifting target domains, we then update this matrix online using the AGOP computed from current high-confidence samples. 
This online refinement allows the subspace to gradually incorporate target-domain information without retraining the classifier, ensuring that adaptation remains both efficient and stable.

For the current test batch \(x_t\) we extract features and compute logits on the adapted features:
\[
 F_t = g_{\phi^\ast}(x_t)\in\mathbb{R}^{B\times L},\qquad
Y_t = h_{\psi}\big(\mathcal{A}(F_t)\big)\in\mathbb{R}^{B\times C}.   
\]

For each sample i in the current batch, let
\[
    p_{t,i}=\max_{c}\ \text{Softmax}(Y_{t})_{i,c},
\]
denote its maximum predicted probability. We then select the high-confidence subset
\[
    \mathcal{M}_t=\{\,i\in[B]\mid p_{t,i}\ge\tau\},
\]
where \(\tau\) is a confidence threshold.

For each \(i\in\mathcal{M}_t\) we form a gradient surrogate given by the gradient of the top logit with respect to the feature:
\[
    g_i \;=\; \nabla_{f_i}\!\Big(\max_{c}\; h_{\psi}(f_i)_c\Big)\in\mathbb{R}^L,
\]
where \(f_i\) may be taken as the detached (pre-adaptation) feature to avoid unnecessary graph propagation.

The per-batch AGOP is constructed as
\[
    \widehat{G}_t^{(b)} \;=\; \frac{1}{|\mathcal{M}_t|}\sum_{i\in\mathcal{M}_t} g_i g_i^\top \in\mathbb{R}^{L\times L},
\]

Batch contributions are aggregated online via an exponential moving average:
\[
    G_t \;=\; (1-\alpha)G_{t-1} + \alpha\,\widehat{G}_t^{(b)},
\]
where \(\alpha\in(0,1]\) controls the update rate.

Every \(T_{\mathrm{eig}}\) batches we perform a symmetric eigendecomposition of \(G_t\) to extract the dominant subspace:
\[
   G_t = Q \Lambda Q^\top,\qquad \Lambda=\operatorname{diag}(\lambda_1,\dots,\lambda_L), 
\]
and select the top-\(r\) eigenvectors to form the subspace basis
\[
    V_t = [v_1,\dots,v_r]\in\mathbb{R}^{L\times r}.
\]
\(V_t\) thus denotes the low-rank subspace to which the adapter is restricted, enabling effective adaptation while substantially reducing parameter and computational overhead.

\begin{table*}[t]
\caption{Online classification error rates (\%) on CIFAR10-C, CIFAR100-C, and ImageNet-C under the CTTA setting. All methods are evaluated online at severity level 5, with our method's results averaged over 5 independent runs. The best performance for each corruption type is highlighted in \textbf{bold} and the next best is \underline{underlined}.}
\label{tab: main}
\centering
\begin{adjustbox}{max width=\textwidth}
\begin{tabular}{@{}ll*{17}{c}@{}}
\toprule
& \textbf{Method} & \textbf{Gau.} & \textbf{shot} & \textbf{imp.} & \textbf{def.} & \textbf{glass} & \textbf{mot.} & \textbf{zoom} & \textbf{snow} & \textbf{fro.} & \textbf{fog} & \textbf{bri.} & \textbf{con.} & \textbf{ela.} & \textbf{pix.} & \textbf{jpeg} & \cellcolor{gray!20}\textbf{Mean} \\
\midrule

\multirow{12}{*}{\rotatebox[origin=c]{90}{\textbf{CIFAR10-C}}}
 & TENT \cite{wang2020tent}       & 24.8 & 20.6 & 28.6 & 14.4 & 31.1 & 16.5 & 14.1 & 19.1 & 18.6 & 18.6 & 12.2 & 20.3 & 25.7 & 20.8 & 24.9 & \cellcolor{gray!20}20.7 \\
 & Ada \cite{chen2022contrastive} & 29.1 & 22.5 & 30.0 & 14.0 & 32.7 & 14.1 & 12.0 & 16.6 & 14.9 & 14.4 & 8.1  & 10.0 & 21.9 & 17.7 & 20.0 & \cellcolor{gray!20}18.5 \\
 & CoTTA \cite{wang2022continual} & 24.3 & 21.3 & 26.6 & 11.6 & 27.6 & \underline{12.2} & 10.3 & 14.8 & 14.1 & 12.4 & 7.6  & 10.6 & 18.3 & 13.4 & 17.3 & \cellcolor{gray!20}16.2 \\
 & RMT \cite{dobler2023robust}    & 24.0 & 20.4 & 25.6 & 12.6 & 25.4 & 14.2 & 12.2 & 15.4 & 15.1 & 14.1 & 10.3 & 13.7 & 17.1 & 13.5 & 16.0 & \cellcolor{gray!20}16.7 \\
 & DSS \cite{wang2024continual}   & 24.1 & 21.3 & 25.4 & 11.7 & 26.9 & \underline{12.2} & 10.5 & 14.5 & 14.1 & 12.5 & 7.8  & 10.8 & 18.0 & 13.1 & 17.3 & \cellcolor{gray!20}16.0 \\
 & EATA \cite{niu2022efficient}   & 24.2 & 19.1 & 27.6 & 12.5 & 29.2 & 14.2 & 12.0 & 16.3 & 15.1 & 15.1 & 8.9  & 13.5 & 21.2 & 16.4 & 21.9 & \cellcolor{gray!20}17.9 \\
 & SAR \cite{niutowards}          & 28.3 & 25.9 & 35.7 & 12.6 & 34.7 & 13.8 & 11.9 & 17.4 & 17.5 & 14.8 & 8.1  & 12.9 & 23.5 & 19.5 & 27.2 & \cellcolor{gray!20}20.3 \\
 & BeCoTTA \cite{lee2024becotta}  & 22.9 & 19.1 & 26.9 & \underline{10.2} & 27.5 & 12.7 & 10.4 & 14.7 & 14.3 & 12.4 & \textbf{7.2} & \underline{9.4} & 20.9 & 15.2 & 20.2 & \cellcolor{gray!20}16.3 \\
 & SANTA \cite{chakrabarty2023santa} & 23.9 & 20.1 & 28.0 & 11.6 & 27.4 & 12.6 & 10.2 & 14.1 & 13.2 & 12.2 & 7.6  & 10.3 & 19.1 & 13.3 & 18.5 & \cellcolor{gray!20}16.1 \\
 & OBAO \cite{zhu2024reshaping}   & \underline{22.2} & \underline{18.3} & \underline{25.0} & 10.7 & \underline{24.7} & \textbf{11.6} & \underline{10.0} & \underline{12.9} & \underline{12.7} & \underline{10.9} & 7.7 & 9.6 & \underline{15.7} & \underline{11.4} & \underline{15.2} & \cellcolor{gray!20}\underline{14.6} \\
 & GOLD                            & \textbf{21.4} & \textbf{18.1} & \textbf{24.5} & \textbf{10.0} & \textbf{24.4} & \textbf{11.6} & \textbf{9.9} & \textbf{12.5} & \textbf{12.3} & \textbf{10.5} & \underline{7.5} & \textbf{8.9} & \textbf{15.3} & \textbf{10.5} & \textbf{14.9} & \cellcolor{gray!20}\textbf{14.1} \\
\midrule

\multirow{14}{*}{\rotatebox[origin=c]{90}{\textbf{CIFAR100-C}}}
 & TENT \cite{wang2020tent}        & 37.2 & 35.8 & 41.7 & 37.9 & 51.2 & 48.3 & 48.5 & 58.4 & 63.7 & 71.1 & 70.4 & 82.3 & 88.0 & 88.5 & 90.4 & \cellcolor{gray!20}60.9 \\
 & Ada \cite{chen2022contrastive}  & 42.3 & 36.8 & 38.6 & 27.7 & 40.1 & 29.1 & 27.5 & 32.9 & 30.7 & 38.2 & 25.9 & 28.3 & 33.9 & 33.3 & 36.2 & \cellcolor{gray!20}33.4 \\
 & CoTTA \cite{wang2022continual}  & 40.1 & 37.7 & 39.7 & 26.9 & 38.0 & 27.9 & 26.4 & 32.8 & 31.8 & 40.3 & 24.7 & 26.9 & 32.5 & 28.3 & 33.5 & \cellcolor{gray!20}32.5 \\
 & RMT \cite{dobler2023robust}     & 40.5 & 36.1 & 36.3 & 27.7 & 33.9 & 28.5 & 26.4 & 29.0 & 29.0 & 32.5 & 25.1 & 27.4 & 28.2 & \underline{26.3} & 29.3 & \cellcolor{gray!20}30.4 \\
 & DSS \cite{wang2024continual}    & 39.7 & 36.0 & 37.2 & 26.3 & 35.6 & 27.5 & 25.2 & 31.4 & 30.0 & 37.8 & 24.2 & 26.0 & 30.0 & \underline{26.3} & 31.3 & \cellcolor{gray!20}30.9 \\
 & EATA \cite{niu2022efficient}    & 37.2 & \textbf{33.0} & 35.9 & 27.7 & 38.1 & 29.4 & 27.1 & 32.8 & 32.7 & 36.1 & 27.0 & 29.4 & 33.8 & 29.5 & 38.2 & \cellcolor{gray!20}32.5 \\
 & SAR \cite{niutowards}           & 40.5 & 34.8 & 37.1 & \textbf{25.6} & 37.2 & 28.0 & 25.5 & 31.9 & 30.9 & 35.8 & 25.2 & 27.8 & 31.8 & 29.0 & 37.2 & \cellcolor{gray!20}31.9 \\
 & BeCoTTA \cite{lee2024becotta}   & 42.1 & 38.0 & 42.2 & 30.2 & 42.9 & 31.7 & 29.8 & 35.1 & 33.9 & 38.5 & 27.9 & 32.0 & 36.7 & 31.6 & 39.9 & \cellcolor{gray!20}35.5 \\
 & SANTA \cite{chakrabarty2023santa} & \textbf{36.6} & \underline{33.2} & \underline{35.1} & \underline{25.9} & 34.9 & 27.7 & 25.4 & 29.5 & 29.9 & 33.1 & \underline{23.6} & 26.7 & 31.9 & 27.5 & 35.2 & \cellcolor{gray!20}30.4 \\
 & OBAO  \cite{zhu2024reshaping}                          & 38.8 & 35.0 & 36.6 & 26.7 & \underline{33.4} & 27.4 & \textbf{25.0} & \underline{28.6} & \underline{28.0} & \underline{31.2} & 24.1 & \textbf{25.1} & \underline{27.5} & \textbf{25.4} & \underline{29.1} & \cellcolor{gray!20}\underline{29.5} \\
 & GOLD                            & \underline{37.1} & 34.0 & \textbf{35.5} & 26.1 & \textbf{32.2} & \textbf{26.7} & \textbf{24.8} & \textbf{27.1} & \textbf{27.0} & \textbf{29.5} & \textbf{23.5} & \underline{25.2} & \textbf{26.8} & \textbf{24.6} & \textbf{28.3} & \cellcolor{gray!20}\textbf{28.6} \\
\midrule

\multirow{13}{*}{\rotatebox[origin=c]{90}{\textbf{ImageNet-C}}}
 & TENT \cite{wang2020tent}        & 81.6 & 74.6 & 72.7 & 77.6 & 73.8 & 65.5 & 55.3 & 61.6 & 63.0 & 51.7 & 38.2 & 72.1 & 50.8 & 47.4 & 53.3 & \cellcolor{gray!20}62.6 \\
 & Ada \cite{chen2022contrastive}  & 82.9 & 80.9 & 78.4 & 81.4 & 78.7 & 72.9 & 64.0 & 63.5 & 64.5 & 53.5 & 38.4 & 66.7 & 54.6 & 49.4 & 53.0 & \cellcolor{gray!20}65.5 \\
 & CoTTA \cite{wang2022continual}  & 84.7 & 82.1 & 80.6 & 81.3 & 79.0 & 68.6 & 57.5 & 60.3 & 60.5 & 48.3 & 36.6 & 66.1 & \textbf{47.3} & 41.2 & 46.0 & \cellcolor{gray!20}62.7 \\
 & RMT \cite{dobler2023robust}     & 80.2 & 76.4 & 74.5 & 77.1 & 74.4 & 66.2 & 57.6 & 57.0 & \underline{59.1} & 48.0 & 39.1 & \underline{60.6} & \textbf{47.3} & 42.5 & \textbf{43.4} & \cellcolor{gray!20}60.2 \\
 & DSS \cite{wang2024continual}    & 82.3 & 78.4 & 76.7 & 81.9 & 77.8 & 66.9 & 60.9 & \textbf{50.8} & 60.9 & 47.7 & 35.4 & 69.0 & 47.5 & \textbf{40.9} & 46.2 & \cellcolor{gray!20}62.2 \\
 & EATA \cite{niu2022efficient}    & \underline{75.7} & \textbf{66.7} & \textbf{65.4} & \underline{74.1} & \textbf{69.2} & 64.2 & 57.5 & 63.0 & 62.8 & 53.4 & 41.3 & 64.0 & 50.5 & 47.2 & 50.0 & \cellcolor{gray!20}60.4 \\
 & SAR \cite{niutowards}           & 82.0 & 74.0 & 71.9 & 77.4 & 73.6 & 66.1 & 56.2 & 61.3 & 63.1 & 51.2 & 37.1 & 69.0 & 49.8 & 46.0 & 51.2 & \cellcolor{gray!20}62.0 \\
 & BeCoTTA \cite{lee2024becotta}   & 84.1 & 74.3 & 72.2 & 77.4 & 71.9 & \underline{63.4} & \underline{55.1} & 57.2 & 61.2 & 50.7 & 36.4 & 66.1 & 49.2 & 45.6 & 48.4 & \cellcolor{gray!20}60.9 \\
 & SANTA \cite{chakrabarty2023santa} & \textbf{74.1} & 72.9 & 71.6 & 75.7 & 74.1 & 64.2 & 55.5 & \underline{55.6} & 62.9 & \textbf{46.6} & \textbf{34.1} & 69.9 & 50.6 & 44.3 & 48.5 & \cellcolor{gray!20}60.1 \\
 & OBAO \cite{zhu2024reshaping}                            & 78.5 & 75.3 & 73.0 & 75.7 & 73.1 & 64.5 & 56.0 & 55.8 & \textbf{58.1} & \underline{47.6} & 40.5 & \textbf{60.1} & \underline{47.4} & 43.9 & \underline{45.8} & \cellcolor{gray!20}\underline{59.6} \\
 & GOLD                            & 76.6 & \underline{72.4} & \underline{70.9} & \textbf{73.3} & \underline{71.4} & \textbf{62.2} & \textbf{54.9} & 57.7 & 60.3 & 48.5 & 38.7 & 61.8 & 48.6 & 45.3 & 47.8 & \cellcolor{gray!20}\textbf{59.3} \\
\bottomrule
\end{tabular}
\end{adjustbox}
\end{table*}

\subsection{Scaling-Vector Update}
\label{sec:4.4}

To provide more reliable supervision for updating the scaling vector, we follow recent CTTA works \cite{zhu2024reshaping,zhang2025analytic,liu2025efficient} and adopt an EMA teacher model. 
Let the student logits for the current batch be $Y_t = h_{\psi}\big(\mathcal{A}(g_{\phi^\ast}(x_t))\big)$ and let the teacher logits be $Y_t^{\mathrm{ema}} = h_{\psi}^{\mathrm{ema}}\big(g_{\phi^\ast}(x_t)\big)$. We also denote by \(Y_t^{+}=h_{\psi}\big(\mathcal{A}(F_t^{+})\big)\) the logits computed on an augmented view \(x_t^{+}\) of the same batch. Using both the original and an augmented view encourages the model to produce predictions that are invariant to realistic input perturbations and thus reduces overfitting to spurious features.

\noindent\textbf{Self-training consistency loss.} The EMA teacher provides a stable target for both the original and augmented views. We use a SCE (Symmetric Cross Entropy) loss \cite{wang2019symmetric}:
\[
    \mathcal{L}_{\mathrm{st}} \;=\; \tfrac{1}{2}\,\mathrm{SCE}\big(Y_t,\,Y_t^{\mathrm{ema}}\big)
\;+\; \tfrac{1}{2}\,\mathrm{SCE}\big(Y_t^{+},\,Y_t^{\mathrm{ema}}\big).
\]
Intuitively, the first term enforces that the student prediction on the original image matches the stable EMA target, while the second term enforces that an augmented view produces the same stable prediction. These terms improve robustness and reduce confirmation bias from noisy pseudo-labels.

\noindent\textbf{Prototype-based contrastive loss.} To further anchor adapted features to source semantics, we form a prototype-based contrastive objective. For each sample \(i\) we find the nearest source prototype $k(i)$ by cosine similarity.
Using the triplet \([P_{k(i)}, f_i, f_i^{+}]\) (where \(f_i^{+}\) is the feature of the augmented view), we encourage both the original and augmented features to align with the chosen prototype while being separated from other prototypes. A practical instantiation uses an InfoNCE-style loss \cite{wang2021understanding,parulekar2023infonce} averaged over the two views:
\begin{align}
\mathcal{L}_{\mathrm{cont}} & =  -\frac{1}{2|\mathcal{B}|}\sum_{i\in\mathcal{B}}
\Bigg[ \log\frac{\exp\big(\mathrm{sim}(f_i,P_{k(i)})/\kappa\big)}
{\sum_{c}\exp\big(\mathrm{sim}(f_i,P_c)/\kappa\big)} \nonumber \\
& + \log\frac{\exp\big(\mathrm{sim}(f_i^{+},P_{k(i)})/\kappa\big)}
{\sum_{c}\exp\big(\mathrm{sim}(f_i^{+},P_c)/\kappa\big)}
\Bigg], \nonumber 
\end{align}
where \(\mathrm{sim}(u,v)=\frac{u^\top v}{\|u\|\|v\|}\) is cosine similarity, \(\kappa>0\) is a temperature, and \(\mathcal{B}\) denotes the set of samples used for contrastive supervision. This loss pulls both views toward the same source prototype and discourages drifting away from source semantics; using the augmented view in the pair improves invariance to input transformations.

\noindent\textbf{Total loss and online update.} The overall objective combines the self-training term and the prototype contrastive term (weighted by tunable coefficients):
\[
    \mathcal{L} \;=\; \lambda_{\mathrm{trg}}\mathcal{L}_{\mathrm{st}} \;+\; \lambda_{\mathrm{cont}}\mathcal{L}_{\mathrm{cont}}.
\]
We perform a \textbf{single gradient step} on the scaling vector \(S_t\) and a small set of batch-norm parameters \(\theta\).

%% file: sec/5_exp.tex
\section{Experiment}

We evaluate the performance of GOLD on both image classification and semantic segmentation tasks, while also assessing its adaptation efficiency.  
\textit{Additional experiments presented in the supplementary material} include parameter robustness analysis, performance across different batch sizes, and additional factors.

\subsection{Dataset and Settings}

Following \cite{ni2025maintaining,wang2022continual}, we evaluate GOLD on CIFAR10-C, CIFAR100-C, and ImageNet-C \cite{hendrycks2019benchmarking}, which are standard robustness benchmarks containing 15 corruption types at severity levels 1--5. We use the clean training sets of CIFAR10, CIFAR100, and ImageNet \cite{krizhevsky2009learning} as source domains, and their corrupted counterparts as target domains. The model adapts sequentially to all 15 corruption types at severity level 5 in an online manner without domain change notification. For each corruption, we use 10,000 images for the CIFAR-based datasets and 5,000 for ImageNet-C.

We strictly adhere to the CTTA protocol without accessing any source data. Our method is compared against both established and state-of-the-art CTTA approaches, all evaluated online under the same conditions. For all datasets, we use the maximum corruption severity level of 5. Model predictions are generated before adapting to the current test stream.
Following CoTTA, we employ standard pre-trained models as our source models: WideResNet-28 \cite{zagoruyko2016wide} for CIFAR10-C, ResNeXt-29 \cite{xie2017aggregated} for CIFAR100-C, and ResNet-50 \cite{he2016deep} for ImageNet-C. Comprehensive descriptions of baseline methods and implementation details are provided in the supplementary material.

\begin{figure}[t]
    \centering
    \includegraphics[width=0.85\linewidth]{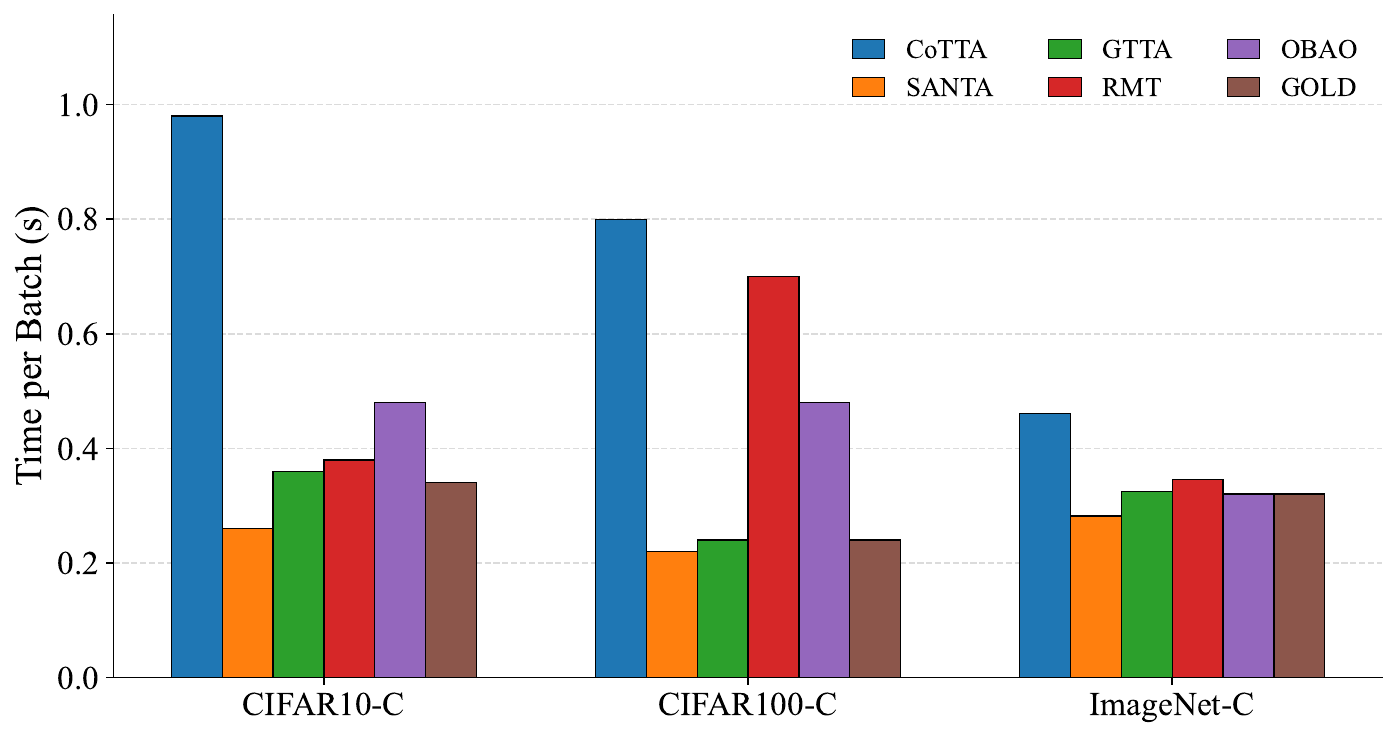}
    \caption{Comparison of method efficiency. The vertical axis shows the time required for each method to perform adaptation and prediction in each test batch.}
    \label{fig:time_bar}
\end{figure}

\subsection{Main Results}

Table \ref{tab: main} presents the classification error rates (lower is better) of various methods under the CTTA setting on three benchmark datasets, all evaluated at the highest corruption severity level 5. Our proposed GOLD establishes state-of-the-art performance across all benchmarks, demonstrating consistent effectiveness in handling diverse corruption types. On CIFAR10-C, GOLD achieves a mean error rate of 14.1\%, outperforming all competing methods and demonstrating particularly strong results on geometrically challenging corruptions like defocus blur and motion blur. The advantage is more pronounced on CIFAR100-C, where GOLD substantially surpasses CoTTA and TENT, indicating superior capability in handling fine-grained classification tasks. For ImageNet-C, GOLD maintains this competitive edge with the best overall performance, showing consistent robustness across datasets of varying scales. These results collectively validate GOLD's effectiveness in diverse continuous test-time adaptation scenarios.

As shown in \Cref{tab: aba}, we conduct an ablation study to investigate the effects of golden subspace estimation and loss components. For subspace construction, we compare three variants: (1) without subspace projection, (2) using the eigendecomposition of \(W^{\top}W\) for initialization (as motivated by \Cref{eq: 1}), and (3) employing the proposed AGOP-based online update. The results demonstrate that subspace projection effectively suppresses excessive parameter updates and stabilizes adaptation. Using the eigenspace of \(W^{\top}W\) as a naive initialization validates our theoretical insight and serves as a fast, stable warm start for adaptation. Furthermore, the AGOP-based online update continuously refines the golden subspace, enabling the model to dynamically align with evolving domain statistics and achieve superior long-term performance. For the scaling-vector update, we also evaluate the impact of the contrastive loss \(\mathcal{L}_{\mathrm{cont}}\), which further enhances model robustness and helps preserve the source-domain semantic structure, effectively mitigating feature drift during continual adaptation.

\Cref{fig:time_bar} shows the per-batch processing time of different methods across various benchmarks.  
Even as the test data becomes more challenging, GOLD consistently maintains an average runtime around \textbf{0.25} seconds, which is comparable to the fastest existing method SANTA \cite{chakrabarty2023santa}, while achieving substantially better performance.

\begin{table}[t]
\caption{
Ablation studies of two components ($G_t$ and $S_t$) on the CIFAR10-C, CIFAR100-C, and ImageNet-C datasets.
}
\label{tab: aba}
\resizebox{\linewidth}{!}{%
\begin{tabular}{cccccc}
\hline
$W^TW$ init & AGOP & $\mathcal{L}_{\mathrm{cont}}$ & CIFAR10-C & CIFAR100-C & ImageNet-C \\ \hline
-           & -           & -                             & 16.24     & 29.87      & 62.45      \\
$\surd$     & -           & -                             & 15.11     & 29.15     & 60.24      \\
$\surd$     &   -   & $\surd$                       & 14.64     & 28.63  & 60.01      \\
$\surd$     & $\surd$     & -                       & 15.06     &  28.77          & 59.87      \\
- & $\surd$ & $\surd$ & 14.32 & 28.61 & 59.52\\
$\surd$     & $\surd$     & $\surd$                       & 14.12     & 28.56   & 59.32     \\ \hline
\end{tabular}
}

\end{table}

\subsection{Experiments on Segmentation CTTA}
\begin{figure*}[t]
    \centering
    \includegraphics[width=0.95\linewidth]{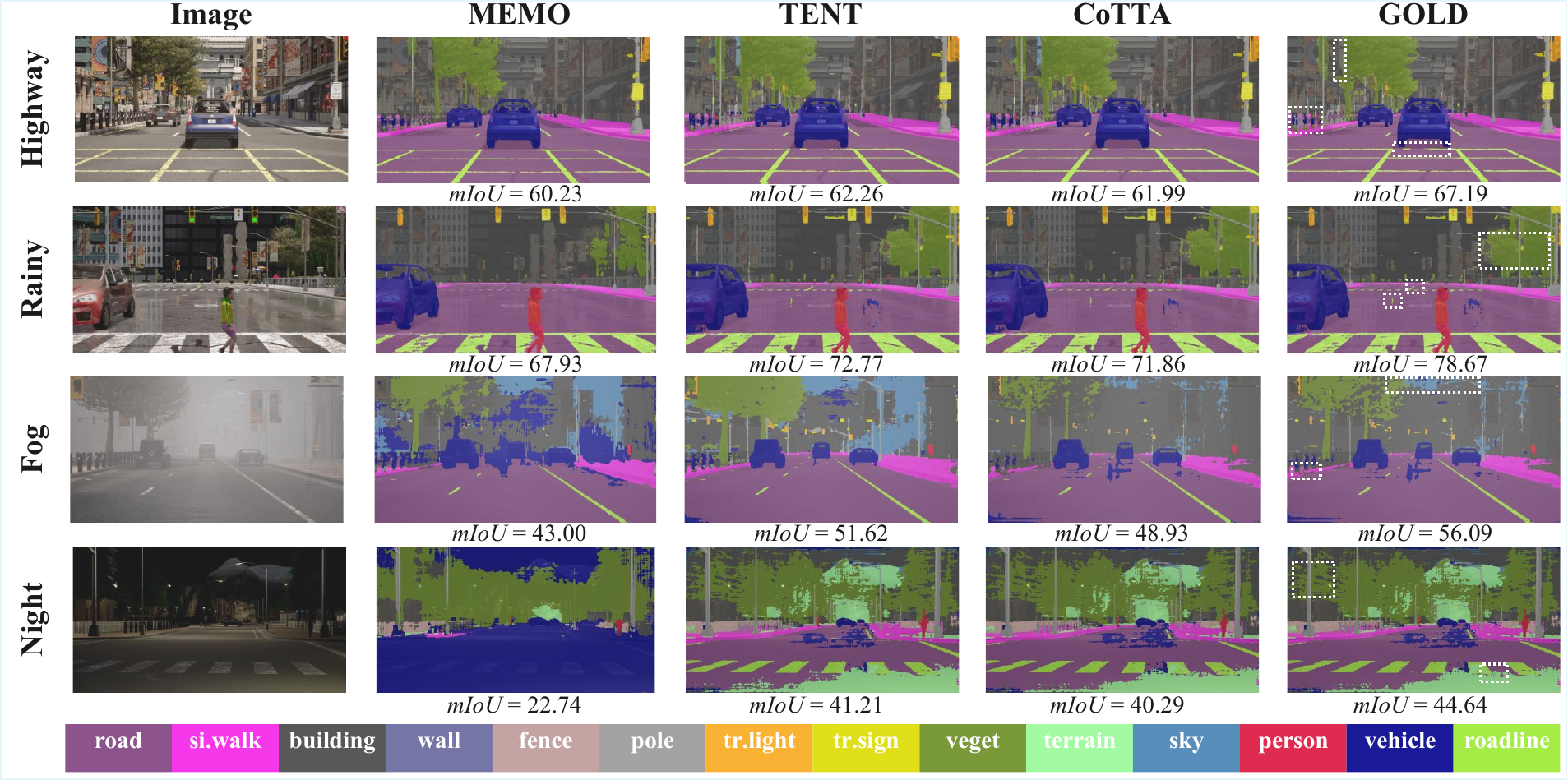}
    \caption{A visual comparison of segmentation results from different adaptive methods on the CarlaTTA sequence is presented. White boxes indicate areas where GOLD has made progress, and magnified images are shown in the supplementary material.}
    \label{fig: seg}
    
\end{figure*}

Under the CarlaTTA benchmark \cite{marsden2024introducing}, a synthetic dataset built on the CARLA simulator \cite{dosovitskiy2017carla} for evaluating gradual test-time adaptation in urban scene segmentation, we simulate five types of continuously evolving and unknown real-world environmental conditions. These include: \textit{day2night} (day to night), \textit{clear2fog} (clear to foggy), \textit{clear2rain} (clear to rainy), \textit{dynamic} (combined varying conditions), and \textit{highway} (transition from urban to highway scenes, which also introduces label distribution shift). The objective of this experiment is to assess the model’s online adaptation capability under these complex and continuously changing driving scenarios. Evaluation is performed by computing the \textbf{mean Intersection-over-Union} (mIoU) over the entire test sequence. \textit{Detailed experimental configurations are provided in the supplementary material.}

Quantitative results are presented in \Cref{tab: seg}. Compared with strong baselines including Source, MEMO, TENT, and CoTTA, our proposed GOLD method achieves competitive performance across all five domain shift scenarios. Specifically, GOLD obtains the highest mIoU on three sequences: \textit{day2night}, \textit{clear2fog}, and \textit{highway}. Notably, on the challenging \textit{highway} sequence which involves both covariate and label distribution shifts, GOLD outperforms all other methods by a clear margin. While CoTTA shows strong results on \textit{highway}, GOLD achieves a further improvement of 0.7\%. These results demonstrate that GOLD, as a lightweight approach, achieves performance comparable to CoTTA while being more computationally efficient.

\textbf{Visual comparisons} are provided in \Cref{fig: seg}, which showcases segmentation examples under different adaptation methods. GOLD captures fine-grained details that previous methods fail to recover. For instance, in the \textit{highway} scenario, GOLD accurately identifies road markings partially occluded by vehicle shadows, while in the \textit{clear2rain} sequence it preserves lane boundaries even under heavy rain streaks. Moreover, in adverse visibility conditions such as \textit{clear2fog} and \textit{day2night}, GOLD exhibits improved recognition of distant scene structures and stable segmentation quality. This advantage arises from the AGOP-guided adaptation process, which rapidly aligns feature representations along important gradient directions, enabling the model to perform fast and effective updates without overfitting transient noise. Consequently, GOLD achieves more reliable adaptation to evolving visual domains while preserving the underlying semantic structure of the scene.

\begin{table}[t]
\caption{Quantitative results of semantic segmentation under gradual domain shifts on CarlaTTA. We report mIoU (\%) on five test sequences. Best results are in \textbf{bold}, second best are \underline{underlined}.}
\label{tab: seg}
\resizebox{\linewidth}{!}{%
\begin{tabular}{cccccc}
\hline
       & day2night     & clear2fog     & clear2rain    & dynamic       & highway       \\ \hline
Source & 58.4          & 52.8          & \textbf{71.8} & 46.2          & 24.7          \\
MEMO \cite{zhang2022memo}  & 61.0          & 55.1          & \underline{71.6}    & 36.7          & 23.6          \\
TENT \cite{wang2020tent} & \underline{61.5}    & 56.0          & 70.9          & \textbf{50.8} & 32.8          \\
CoTTA \cite{wang2022continual} & 61.4          & 56.8          & 70.7          & 46.3          & \underline{33.8}    \\
GOLD   & \textbf{61.8} & \textbf{57.1} & 71.0          & \underline{48.9}    & \textbf{34.5} \\ \hline
\end{tabular}%
}

\end{table}

%% file: sec/6_con.tex
\section{Conclusion} 
We presented GOLD, a source-free CTTA framework that performs adaptation through structured feature updates within a low-rank golden subspace. By combining an online AGOP-based subspace estimate with a lightweight scaling adapter, GOLD achieves a strong balance between adaptation efficiency and robustness across classification and segmentation benchmarks. We hope this perspective motivates future work on structure-aware online adaptation beyond global parameter updates.

%% file: sec/X_suppl.tex
\clearpage
\section*{Appendix}

\setcounter{table}{0}
\renewcommand{\thetable}{S\arabic{table}}
\setcounter{figure}{0}
\renewcommand{\thefigure}{S\arabic{figure}}
\setcounter{section}{0}
\renewcommand{\thesection}{\Alph{section}}

The supplementary material is organized as follows.
\begin{itemize}
    \item Appendix \ref{sec:s1} introduces the notation used throughout the paper.
    \item Appendix \ref{sec:s2} provides the full theoretical proof of the existence of the golden subspace.
    \item Appendix \ref{sec:s3} presents additional experimental details and results, including the detailed experimental setup, analyses under different batch sizes and hyperparameter settings, comprehensive efficiency analysis, additional qualitative segmentation results, and long-term generalization experiments.
\end{itemize}

\section{Notations}\label{sec:s1}

To facilitate reading, \Cref{tab:notation} provides the mathematical symbols that appear throughout the text and their meanings.

\begin{table}[ht]
\centering
\small
\caption{Notation Summary.}
\label{tab:notation}
\resizebox{\linewidth}{!}{%
\begin{tabular}{cc}
\hline
\textbf{Symbol} & \textbf{Meaning} \\
\hline
$D_S$ & Labeled source-domain dataset \\
$D_T$ & Unlabeled target-domain data stream \\
$X_t$ & Mini-batch of target samples at step $t$ \\
$g_\theta$ & Full model composed of feature extractor and classifier \\
$g_\phi$ & Pretrained feature extractor (frozen backbone) \\
$h_\psi$ & Classifier head producing logits \\
$F_t$ & Feature matrix of batch $t$ \\
$f$ & Feature vector of a single sample \\
$W$ & Linear classifier weight matrix \\
$P_c$ & Prototype of class $c$ in the source domain \\
$P$ & Matrix of all class prototypes \\
$V_t$ & Basis of the adaptive subspace (top-$r$ eigenvectors) \\
$S_t$ & Learnable scaling vector in the subspace \\
$u$ & Coordinates of feature $f$ in the subspace \\
$\tilde{u}$ & Scaled coordinates after subspace adaptation \\
$A(f)$ & Adapter transformation applied to feature $f$ \\
$F_t^{\text{adapt}}$ & Adapted feature matrix for batch $t$ \\
$\hat{G}^{(b)}_t$ & Mini-batch estimate of the AGOP matrix \\
$g_i$ & Gradient proxy with respect to feature $f_i$ \\
$M_t$ & Set of high-confidence samples in batch $t$ \\
$G_t$ & Exponential moving average (EMA) of AGOP \\
$Q,\Lambda$ & Eigenvectors and eigenvalues of $G_t$ \\
$\kappa(k)$ & Cumulative spectral energy ratio \\
$Y_t$ & Student model logits \\
$Y_t^{ema}$ & EMA teacher logits \\
$Y_t^+$ & Logits from augmented views \\
$\mathcal{L}_{st}$ & EMA consistency loss \\
$\mathcal{L}_{cont}$ & Prototype-based contrastive loss \\
$\mathcal{L}$ & Total training objective \\
$\theta^{ema}$ & Parameters of the EMA teacher model \\
$\alpha, \tau, r$ & EMA decay, confidence threshold, subspace rank \\
$m$ & Momentum coefficient for EMA updates \\
\hline
\end{tabular}
}
\end{table}

\section{Detailed Analysis of GOLD}\label{sec:s2}

\subsection{Proof: Existence of the golden subspace}

We begin by formalizing the claim that the minimal feature perturbation required to produce a prescribed logit change must lie in the row-space of the linear classifier. This statement underlies the notion of a ``golden subspace'' for test-time feature adaptation.

\begin{assumption}[Linear classifier]
The classification head is linear in features: for any feature vector \(f\in\mathbb{R}^L\) the pre-softmax logit vector satisfies
\[
z = W f,
\]
where \(W\in\mathbb{R}^{C\times L}\) is the classifier weight matrix (the last fully-connected layer).
\end{assumption}

\begin{proposition}[Minimal-norm feature change]
Given a desired logit change \(\Delta y\in\mathbb{R}^C\), the constrained minimization
\[
\min_{\Delta f\in\mathbb{R}^L}\ \tfrac{1}{2}\|\Delta f\|_2^2
\quad\text{s.t.}\quad W\Delta f=\Delta y
\]
has the unique minimal-norm solution
\[
\Delta f^\star = W^{+}\Delta y,
\]
where \(W^{+}\) denotes the Moore--Penrose pseudoinverse of \(W\). In particular, \(\Delta f^\star\in\mathrm{row}(W)\), i.e. the column space of \(W^\top\).
\end{proposition}

\begin{proof}
Form the Lagrangian
\[
\mathcal{L}(\Delta f,\lambda)=\tfrac{1}{2}\|\Delta f\|_2^2 + \lambda^\top (W\Delta f-\Delta y),
\qquad \lambda\in\mathbb{R}^C.
\]
Stationarity in \(\Delta f\) yields \(\Delta f + W^\top\lambda = 0\), hence \(\Delta f = -W^\top\lambda\).
Inserting this into the constraint gives \(-(W W^\top)\lambda=\Delta y\). Solving (via the pseudoinverse on the potentially rank-deficient matrix \(W W^\top\)) yields
\[
\lambda = -(W W^\top)^{+}\Delta y,
\]
and therefore
\[
\Delta f^\star = -W^\top (W W^\top)^{+}\Delta y = W^{+}\Delta y,
\]
using the identity \(W^{+}=W^\top (W W^\top)^{+}\). Equivalently, if \(W=U\Sigma V^\top\) is a compact SVD, then
\[
\Delta f^\star = V\Sigma^{+} U^\top \Delta y,
\]
which has no component in \(\ker(W)\) and is the unique minimal-norm feasible vector. Uniqueness follows because any feasible \(\Delta f\) can be decomposed as \(\Delta f=\Delta f^\star+v\) with \(v\in\ker(W)\), and \(\|\Delta f\|_2^2=\|\Delta f^\star\|_2^2+\|v\|_2^2\ge\|\Delta f^\star\|_2^2\).
\end{proof}

\noindent\textbf{Batch extension.}
The same argument applies columnwise: for a batch of \(B\) samples let \(\Delta Y\in\mathbb{R}^{C\times B}\) and \(\Delta F\in\mathbb{R}^{L\times B}\). The Frobenius-norm constrained problem
\[
\min_{\Delta F}\ \tfrac{1}{2}\|\Delta F\|_F^2\quad\text{s.t.}\quad W\Delta F=\Delta Y
\]
has the solution \(\Delta F^\star = W^{+}\Delta Y\). Thus the minimal collective feature change again lies in the row-space of \(W\).

\vspace{6pt}
The implication is direct: to realize logit (or prediction) corrections while minimally perturbing features, it suffices to search for updates inside the subspace spanned by the rows of \(W\) (equivalently the columns of \(W^\top\)). We refer to this subspace as the \emph{golden subspace}. 
\begin{figure*}[t]
    \centering
    \includegraphics[width=1.00\linewidth]{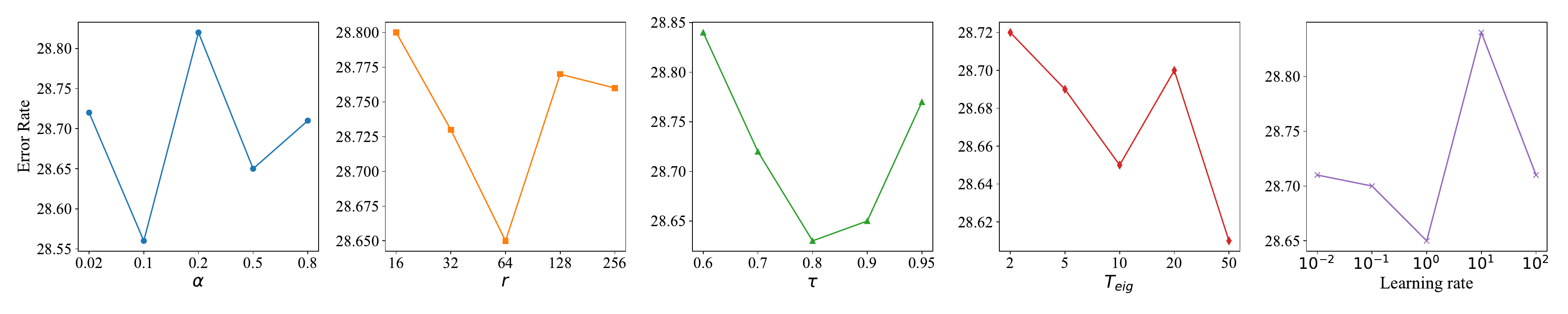}
    \caption{Hyper-parameter robustness: representative results from the sweep described in Section~\ref{sec:hyper}. Each subplot shows the evaluated metric as a function of a single hyper-parameter while other hyper-parameters are held near their default values.}
    \label{fig:hyper}
\end{figure*}
\begin{figure}[t]
    \centering
    \includegraphics[width=1.00\linewidth]{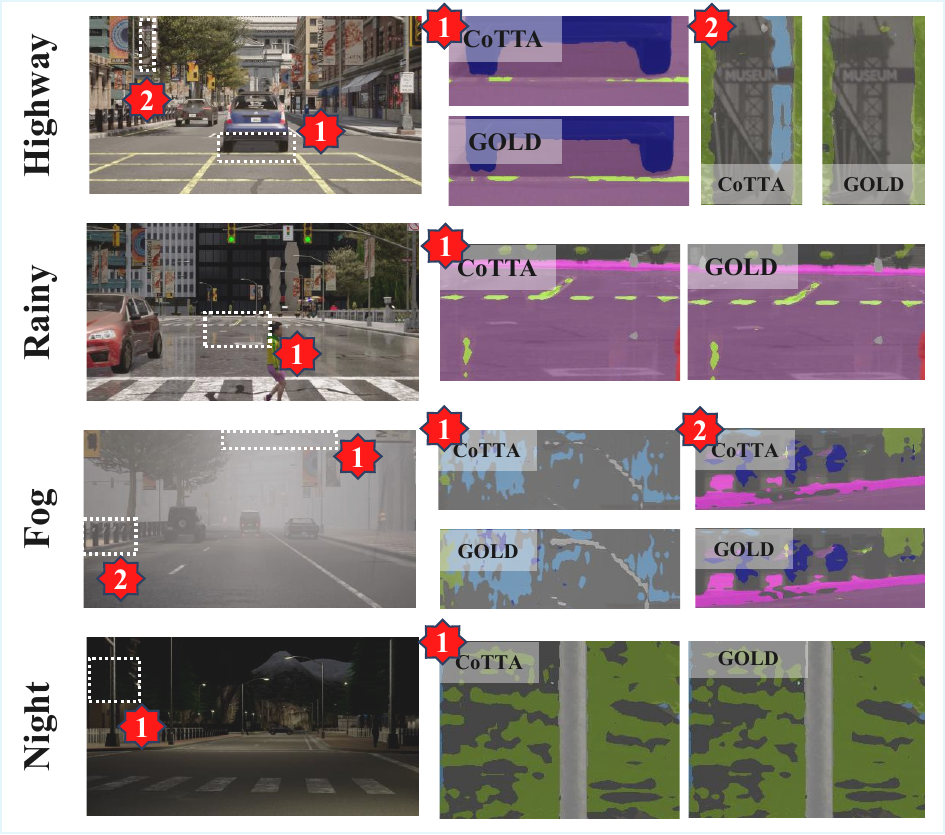}
    \caption{Detailed segmentation comparisons (cropped views). White boxes indicate regions shown at higher resolution to emphasize the qualitative gains of GOLD in challenging conditions (shadows, reflections, fog/night).}
    \label{fig:segdtl}
\end{figure}

\subsection{Adapter Design and Pseudo-Label Robustness}

\label{app:adapter_pseudolabel}

\paragraph{Why an adapter?}
We adopt an adapter-style design to meet the core constraint of CTTA: \emph{online adaptation must be both efficient and stable under non-stationary test streams}. Updating the full backbone (or a large fraction of parameters) can improve short-term fitting but often harms efficiency and exacerbates drift/forgetting as the stream evolves. In contrast, a lightweight adapter provides a structured and low-dimensional update interface: it preserves most pretrained representations while allowing the model to quickly adjust along a compact set of directions that are most relevant to the current target shift.
This design choice is aligned with continual learning insights that emphasize \emph{controlling the effective update subspace} to improve robustness to stream order and to reduce misleading gains that arise from unstable adaptation dynamics \cite{lai2025order,zhou2024expandable,cao2025erroreraser,cao2024open}. In GOLD, the adapter is further constrained to operate in a low-rank \emph{golden subspace}, so that adaptation is concentrated on a minimal set of feature directions rather than dispersed across the entire representation space.

\paragraph{Robustness of confidence-based pseudo labels for AGOP.}
AGOP relies on pseudo labels derived from the model's predictions to form a sample-wise gradient outer product estimate. A natural concern is that noisy pseudo labels may corrupt the estimated directions and lead to drift. We justify the robustness of our design from two complementary aspects.

\noindent \textbf{(i) Built-in safeguards.}
GOLD includes three mechanisms that jointly stabilize pseudo-label-based estimation.
\emph{(a) Warm-start initialization.} We initialize the subspace estimate with the theoretically motivated prior $G_0 = W^\top W$ (Sec.~3.2), which anchors early updates to the pretrained classifier geometry when the model is least adapted and pseudo labels are potentially noisier.
\emph{(b) EMA smoothing with self-training consistency.} We use an exponential moving average (EMA) teacher to generate more stable predictions and enforce self-training consistency (Eq.~17), which reduces the variance of pseudo labels across adjacent samples/batches and discourages abrupt direction changes.
\emph{(c) Low-rank/subspace constraint.} By restricting updates to a compact low-rank subspace, the estimator is regularized implicitly: even if individual pseudo labels are imperfect, their influence is projected onto a small set of directions, mitigating error accumulation and reducing forgetting/drift over long streams.

\noindent \textbf{(ii) Empirical sensitivity.}
We further validate robustness by sweeping the confidence threshold used to filter pseudo labels. As shown in Fig.~S1, GOLD remains insensitive across a wide range of thresholds: even with a relatively low threshold, the performance drop is only $\sim$0.2. This suggests that the proposed safeguards effectively prevent noisy pseudo labels from dominating the AGOP-based subspace maintenance, leading to stable adaptation in practice.

\section{Supplementary Experiments}\label{sec:s3}

This section supplements the main paper with additional empirical analyses on (i) the sensitivity of GOLD to \textbf{test-time batch size}, (ii) \textbf{hyper-parameter robustness}, (iii) \textbf{comprehensive efficiency} (trainable parameter ratio, FLOPs, and peak GPU memory), and (iv) additional \textbf{segmentation} results and qualitative comparisons. Unless otherwise specified, all experiments use the same pretrained backbones and CTTA evaluation protocol as in the main text. For corruption benchmarks, we report mean performance over the full benchmark and follow the same metrics and legend definitions as in the main paper.

\subsection{Detailed Experimental Setup}

\paragraph{Datasets.}
We evaluate GOLD on three standard continual test-time adaptation benchmarks: CIFAR-10-C, CIFAR-100-C, and ImageNet-C. 
CIFAR-10-C and CIFAR-100-C are corrupted versions of the original CIFAR-10 and CIFAR-100 test sets, respectively, while ImageNet-C is constructed from the ImageNet validation set with a diverse set of synthetic corruptions. 
Following the standard protocol, we consider 15 corruption types at severity level 5, including \textit{gaussian\_noise}, \textit{shot\_noise}, \textit{impulse\_noise}, \textit{defocus\_blur}, \textit{glass\_blur}, \textit{motion\_blur}, \textit{zoom\_blur}, \textit{snow}, \textit{frost}, \textit{fog}, \textit{brightness}, \textit{contrast}, \textit{elastic\_transform}, \textit{pixelate}, and \textit{jpeg\_compression}. 
We adopt the one-pass online adaptation setting, where each target sample is observed only once in a streaming manner without revisiting previous target data.

\paragraph{Parameter settings.}
For a fair comparison, all methods use the same optimization setup during test-time adaptation. 
We update the affine parameters of normalization layers, including BatchNorm, LayerNorm, and GroupNorm. 
Concretely, for each normalization layer, we optimize only its \texttt{weight} and \texttt{bias} parameters. 
For all methods, we use Adam as the optimizer with learning rate $1\times10^{-3}$, weight decay $0$, and one adaptation step per batch. 
We use Adam with $\texttt{betas}=(0.9, 0.999)$.

\paragraph{Augmentation details.}
For test-time augmentation, we follow a unified transformation pipeline consisting of Gaussian blur, center crop, random horizontal flip, additive Gaussian noise, and clipping to the valid input range. 
Specifically, we use \texttt{GaussianBlur} with kernel size 5 and sigma sampled from $[0.001, 0.25]$ for the soft setting and from $[0.001, 0.5]$ otherwise, followed by \texttt{CenterCrop}, \texttt{RandomHorizontalFlip} with probability 0.5, additive Gaussian noise, and a final clipping operation to $[0,1]$.

\paragraph{Source models.}
Following CoTTA, we employ standard pre-trained source models for different benchmarks: WideResNet-28~\cite{zagoruyko2016wide} for CIFAR10-C, ResNeXt-29~\cite{yin2019fourier} for CIFAR100-C, and ResNet-50~\cite{dobler2023robust} for ImageNet-C. 
These source models are fixed before test-time adaptation and serve as the initialization for all compared methods.

\paragraph{Source prototype extraction.}
We extract source prototypes offline from the source training set before online adaptation. 
Specifically, we first feed all source samples through the feature extractor and collect their intermediate feature representations together with the corresponding class labels. 
For feature tensors with spatial dimensions, we flatten them into one-dimensional vectors. 
Then, for each class, we compute the class prototype as the mean feature vector of all source samples belonging to that class:
\[
\mathbf{p}_c = \frac{1}{N_c}\sum_{i:y_i=c}\mathbf{f}_i,
\]
where $\mathbf{f}_i$ denotes the extracted source feature, $y_i$ is its label, and $N_c$ is the number of source samples in class $c$. 
If a class has no source samples, we assign it a zero prototype vector. 
All source prototypes are computed once before adaptation and kept fixed throughout the online test-time adaptation process.

\subsection{Effect of batch size}\label{sec:bs}

In the main text we reported results obtained with a batch size of 200 on CIFAR-C and 64 on ImageNet-C. Here we study the influence of test-time batch size on performance by evaluating GOLD and several baseline methods at batch sizes \(\{4,16,32,64\}\). Table~\ref{tab:bs1} summarizes the results.

\begin{table*}[t]
  \caption{Performance under different test-time batch sizes. Results are shown for CIFAR-10-C, CIFAR-100-C and ImageNet-C. Columns report the evaluated batch size. The best value in each column (among the listed methods) is typeset in \textbf{bold}.}
  \label{tab:bs1}
  \centering
  \begin{tabular}{lcccc|cccc|cccc}
    \toprule
    & \multicolumn{4}{c}{\textbf{CIFAR10-C}} & \multicolumn{4}{c}{\textbf{CIFAR100-C}} & \multicolumn{4}{c}{\textbf{ImageNet-C}} \\
    \cmidrule(r){2-5} \cmidrule(lr){6-9} \cmidrule(l){10-13}
    \textbf{Method} & 4 & 16 & 32 & 64 & 4 & 16 & 32 & 64 & 4 & 16 & 32 & 64 \\
    \midrule
    TENT \cite{wang2020tent}    & 86.9 & 61.9 & 40.2 & 22.9 & 98.3 & 96.2 & 92.9 & 86.1 & 99.0 & 82.5 & 64.5 & 62.6 \\
    Ada \cite{chen2022contrastive}     & 34.5 & 19.7 & 17.7 & 17.4 & 90.1 & 54.6 & 41.3 & 35.4 & 98.4 & 81.9 & 69.9 & 65.5 \\
    CoTTA \cite{wang2022continual}   & 52.6 & 26.9 & 21.3 & 17.5 & 68.8 & 39.1 & 36.7 & 34.3 & 97.5 & 82.0 & 64.8 & 62.7 \\
    RMT \cite{dobler2023robust}      & 51.4 & 26.5 & 20.0 & 16.7 & 87.4 & 48.7 & 39.1 & 32.7 & 99.4 & 83.4 & 68.7 & 60.2 \\
    SANTA \cite{chakrabarty2023santa} & 39.5 & 19.5 & 17.5 & 16.8 & 58.5 & 35.8 & 33.3 & 32.5 & 98.7 & 68.0 & 62.8 & 60.1 \\
    GOLD (ours)                     & \textbf{31.7} & \textbf{18.9} & \textbf{17.0} & \textbf{16.4} & \textbf{57.8} & \textbf{35.1} & \textbf{32.8} & \textbf{31.9} & \textbf{97.3} & 69.5 & \textbf{62.1} & \textbf{59.3} \\
    \bottomrule
  \end{tabular}
\end{table*}

\paragraph{Discussion.}
The table shows that GOLD is stable across a wide range of batch sizes. On CIFAR-10-C, GOLD attains the best metric among compared methods for all evaluated batch sizes (columns 4--64). On CIFAR-100-C and ImageNet-C GOLD is competitive with state-of-the-art baselines; in particular GOLD achieves very strong performance at medium-to-large batch sizes on ImageNet-C (see bold entries). These results indicate that the proposed AGOP-based subspace estimation and the low-rank adaptation scheme are effective even when the number of high-confidence samples per batch is relatively small. In practice we recommend using batch sizes of at least 16–32 when computational resources allow, though reasonable performance is still obtained with batch size 4.

\subsection{Hyper-parameter robustness}\label{sec:hyper}

We performed an extensive hyper-parameter sweep to assess stability. The swept ranges were:
\begin{itemize}
  \item EMA momentum \(\alpha\in\{0.02,0.1,0.2,0.5,0.8\}\);
  \item retained rank \(r\in\{16,32,64,128,256\}\);
  \item confidence threshold \(\tau\in\{0.6,0.7,0.8,0.9,0.95\}\);
  \item eigendecomposition period \(T_{\mathrm{eig}}\in\{2,5,10,20,50\}\) (in batches);
  \item learning rate for the small adaptive parameter set (scaling vector \(s\) and optional BN adapters) \(\eta\in\{10^{-2},10^{-1},1,10,10^{2}\}\).
\end{itemize}

Figure~\ref{fig:hyper} visualizes representative slices of this sweep (one plot per dataset / metric). The main conclusions are:
\begin{enumerate}
  \item \textbf{Overall robustness.} GOLD exhibits small performance variation over wide ranges of \(\alpha\) and \(\eta\); extreme values (very large \(\eta\) or \(\alpha\)) can destabilize adaptation but are easy to detect.
  \item \textbf{Rank trade-off.} Increasing \(r\) generally improves performance up to a point; in our experiments \(r\) between 32 and 128 provides a good balance between accuracy and computational cost.
  \item \textbf{Confidence threshold.} A moderate threshold \(\tau\) (e.g. \(0.8\)) mitigates pseudo-label bias while retaining enough samples for stable AGOP estimation. Very high thresholds (e.g. \(0.95\)) reduce the effective sample size and may increase variance.
  \item \textbf{Eig update period.} Frequent eigen updates (small \(T_{\mathrm{eig}}\)) track fast distribution shifts but increase compute; values in the range \(T_{\mathrm{eig}}=5\!-\!20\) offered a good practical tradeoff in our workloads.
\end{enumerate}

These observations guided the default hyper-parameter choices used throughout the main experiments and indicate that GOLD can be tuned with modest effort for new datasets or deployment constraints.

\subsection{Comprehensive efficiency analysis.}
We provide a systematic efficiency comparison on CIFAR100-C (Table~\ref{tab:efficiency} and Fig.~4 in the main paper) to quantify the practical cost of online CTTA. All methods are evaluated under the same backbone and input resolution, and we report three complementary metrics: \textbf{(i) trainable parameter ratio}, \textbf{(ii) FLOPs}, and \textbf{(iii) peak GPU memory}. Together, these metrics capture \emph{what is updated} (parameter footprint), \emph{how much computation is incurred per batch} (throughput cost), and \emph{how much hardware budget is required} (deployment feasibility).

\paragraph{Trainable parameter ratio.}
This measures the fraction of parameters receiving gradients during adaptation. Full fine-tuning style CTTA methods (CoTTA, RMT, GTTA) effectively update the whole model, resulting in a 100\% trainable ratio. In contrast, \textbf{SANTA} and \textbf{GOLD} operate with lightweight modules: only a small adapter and/or scaling parameters are optimized, leading to orders-of-magnitude fewer trainable parameters (0.365\% and 0.373\%, respectively). This is particularly important for online deployment where frequent gradient updates must be performed with minimal overhead and minimal risk of representation drift.

\paragraph{FLOPs.}
We report the per-batch floating-point operations to reflect end-to-end adaptation cost, including both forward and backward passes. \textbf{CoTTA} is the most expensive due to full-model updates and multi-view/self-ensemble style operations, reaching 7842.80\,G FLOPs. \textbf{RMT} and \textbf{GTTA} reduce computation compared to CoTTA but still remain substantially heavier than adapter-based methods (1886.02\,G and 2614.27\,G). In contrast, \textbf{SANTA} and \textbf{GOLD} are significantly more efficient (1348.09\,G and 1425.14\,G), showing that restricting updates to a compact subspace can deliver strong performance without the heavy compute of full-parameter adaptation. Notably, GOLD incurs only a modest additional FLOPs over SANTA due to maintaining and applying the golden-subspace projection.

\paragraph{Peak GPU memory.}
Peak memory reflects the maximum activation/gradient/storage footprint during online adaptation. Memory is a key bottleneck for real-time CTTA, especially when running larger backbones or higher resolutions. Full-update methods require storing gradients for most layers, leading to high memory usage (CoTTA: 10.88\,GB; RMT/GTTA: 5.07--5.21\,GB). Adapter-based methods are notably lighter (SANTA: 4.61\,GB), and while \textbf{GOLD} uses slightly more memory (5.37\,GB) due to additional subspace-related buffers, it remains far below full-model adaptation and fits comfortably within typical single-GPU deployment constraints.

\paragraph{Takeaway.}
Table~\ref{tab:efficiency} confirms that \textbf{GOLD achieves a favorable accuracy--efficiency trade-off}: it keeps the trainable parameter ratio below 0.4\% while maintaining compute and memory footprints close to other lightweight methods, yet delivers stronger and more stable adaptation performance (main paper, Fig.~4). This supports our design principle that CTTA should prioritize \emph{structured minimal updates}---adapting only along critical directions---to maximize both generalization and deployability.

\vspace{-0.5em}
\begin{table}[ht]
    \centering
    \caption{Comprehensive efficiency comparison.}
    \vspace{-0.3em}
    \renewcommand{\arraystretch}{0.9}
    \setlength{\tabcolsep}{5pt}
    \label{tab:efficiency}
    \begin{tabular}{lccc}
        \toprule
        Method & \makecell[c]{Trainable parameter\\ratio (\%)} &  \makecell[c]{FLOPs \\  (G)} & \makecell[c]{Peak GPU\\memory (GB)} \\
        \midrule
        CoTTA & 100     & 7842.80 & 10.88 \\
        RMT   & 100     & 1886.02 & 5.07  \\
        GTTA  & 100     & 2614.27 & 5.21  \\
        SANTA & 0.365   & 1348.09 & 4.61  \\
        GOLD  & 0.373   & 1425.14 & 5.37  \\
        \bottomrule
    \end{tabular}
\end{table}
\vspace{-0.5em}

\paragraph{Runtime breakdown.} To further clarify the practical overhead introduced by GOLD, we profiled the wall-clock time of its main components during online adaptation. On batches where eigendecomposition is performed, the forward and backward passes dominate the runtime, accounting for 45.7\% and 50.8\% of the total time, respectively. In contrast, the additional overhead introduced by AGOP computation is \textbf{only 0.3\%}, while periodic eigendecomposition contributes 3.3\%. These results indicate that the extra cost of GOLD mainly comes from the standard optimization steps shared by most adaptation methods, whereas the proposed subspace maintenance itself adds only a small overhead. This supports our claim that GOLD achieves strong adaptation performance with a favorable efficiency profile.

\subsection{Expanded segmentation examples}

We provide enlarged views of selected segmentation examples to highlight qualitative improvements obtained by GOLD. Figure~\ref{fig:segdtl} shows cropped regions (white bounding boxes) that emphasize the areas of largest visual difference between the pretrained baseline and the adapted model. Representative cases include:
\begin{itemize}
  \item \textbf{Shadowed zebra crossings (Highway).} GOLD recovers occluded lane markings in regions where strong shadows cover parts of the road surface.
  \item \textbf{Puddles and water reflections (Rain).} The adapted model more reliably segments road surfaces and curbs in the presence of specular highlights and reflections.
  \item \textbf{Fog and low-visibility (Fog / Night).} GOLD improves detection of distant objects and preserves fine-grained foreground structure under heavy atmospheric degradation.
\end{itemize}
\subsection{Long-term Generalization under Repeated Domain Exposure}

\paragraph{Setup.}
To evaluate long-term generalization, we conduct an extended continual test-time adaptation experiment on CIFAR10-C. 
Instead of passing through the target stream only once, we repeatedly expose the model to the same corrupted target data for 10 rounds in sequence. 
This setting allows us to examine whether an adaptation method remains stable over prolonged online updates, or gradually suffers from error accumulation and representation drift. 
In particular, it reflects the method's ability to preserve robust generalization under repeated domain exposure, which is crucial for realistic deployment scenarios where incoming data may exhibit persistent or recurring shifts over a long time horizon.

\paragraph{Results.}
Table~\ref{tab:longterm} reports the average error rates of different methods under the 10-round continual adaptation setting on CIFAR10-C. 
Compared with prior CTTA methods, GOLD achieves the best performance with an error rate of 14.15\%, showing stronger stability and better long-term generalization under sustained adaptation. 
These results suggest that restricting updates to the dynamically estimated golden subspace can effectively mitigate parameter drift and maintain adaptation quality over long horizons.

\begin{table}[ht]
    \centering
    \caption{Long-term generalization comparison on CIFAR10-C under 10 rounds of continual adaptation. Lower is better.}
    \vspace{-0.3em}
    \renewcommand{\arraystretch}{1.0}
    \setlength{\tabcolsep}{6pt}
    \label{tab:longterm}
    \begin{tabular}{lcccc}
        \toprule
        Method & CoTTA & RMT & BeCoTTA & SANTA \\
        Error (\%) & 16.21 & 16.73 & 16.28 & 16.29 \\
        \midrule
        Method & EATA & SAR & ViDA & \textbf{GOLD} \\
        Error (\%) & 17.84 & 20.31 & 17.24 & \textbf{14.15} \\
        \bottomrule
    \end{tabular}
\end{table}